\documentclass[11pt]{article}
\usepackage[T1]{fontenc}
\usepackage[margin=1in]{geometry}
\usepackage{amsmath,amssymb,amsthm,mathtools}
\usepackage{xcolor}
\usepackage{float}
\floatstyle{ruled}
\newfloat{algorithm}{ht}{loa}
\floatname{algorithm}{Algorithm}
\usepackage[numbers,sort&compress]{natbib}
\usepackage{hyperref}

\hypersetup{
    colorlinks=true,
    linkcolor=blue,
    citecolor=blue,
    urlcolor=blue
}

\newtheorem{theorem}{Theorem}[section]
\newtheorem{lemma}[theorem]{Lemma}

\theoremstyle{definition}

\numberwithin{equation}{section}

\long\def\rev#1{#1}

\title{Prediction with Expert Advice: Anytime Regret with\\ Many Experts Matches the Fixed-Time Constant\thanks{The main result was entirely obtained by Cogentic, an agentic framework for mathematical discovery, using an interval version of Gemini as the base model. The authors contextualized the findings and verified the proofs. The full exposition here is due to the authors aided by different AI models.

The following authors have additional affiliations beyond Google Research: Yang Cai (Yale University) and Vineet Gupta (Google DeepMind).}}
\author{
\parbox{0.95\textwidth}{\centering
Yang Cai,
Vineet Gupta,
Yanchen Jiang,
Christopher Liaw,\\[0.3em]
Aranyak Mehta,
Grigoris Velegkas,
Di~Wang
}\\[1.2em]
\normalsize Google Research\\[0.3em]
\small\texttt{\{caiy, vineet, yanchenjiang, cvliaw, aranyak, gvelegkas, wadi\}@google.com}
}
\date{}

\begin{document}

\maketitle

\begin{abstract}
Prediction with expert advice is a fundamental problem in online learning. When the time horizon $T$ is known in advance, the minimax cumulative regret over $n$ experts is asymptotically $\sqrt{\frac{T \ln n}{2}}$.
This is achieved by the Multiplicative Weights Update algorithm with a learning rate tuned to $T$, and is known to be tight. If instead the regret bound is required to hold    simultaneously at every time $t$, the best known guarantee has been $\sqrt{t \ln n}$---a factor of $\sqrt{2}$ worse---and it has remained unknown whether this factor of $\sqrt{2}$ is necessary. We show that it is not. We give an algorithm, requiring no knowledge of the horizon, whose cumulative regret satisfies $R_t \le \bigl(1 + O(\sqrt{\ln \ln n / \ln n})\bigr)\sqrt{t \ln n / 2}$ simultaneously for every $t \ge 1$.
% As $n \to \infty$ this matches the fixed-time guarantee, so advance knowledge of the horizon confers no advantage in the leading constant.
\end{abstract}

\section{Introduction}\label{sec:intro}

Sequential prediction with expert advice is a foundational problem in online learning, game theory, and sequential decision-making~\cite{Vovk90, LW94, FS97, CBFHHSW97, Vovk98, CBL06}. At each round $t = 1, 2, \dots$, a learner selects a probability distribution $p_t$ in the probability simplex $\Delta_n$ over $n$ experts $[n] := \{1, \dots, n\}$, observes an adversarial loss vector $\ell_t \in [0, 1]^n$, and incurs expected loss $\langle p_t, \ell_t \rangle = \sum_{i=1}^n p_{t,i} \ell_{t,i}$. The learner's objective is to minimize cumulative regret $R_T = \sum_{t=1}^T \langle p_t, \ell_t \rangle - \min_{i \in [n]} \sum_{t=1}^T \ell_{t,i}$ relative to the best single expert in hindsight over $T$ rounds. When the time horizon $T$ is \emph{known in advance} (the fixed-horizon setting), running the Multiplicative Weights Update algorithm with a constant learning rate tuned to $T$ guarantees $R_T \le \sqrt{\frac{T \ln n}{2}}$ for all $n \ge 2$ and $T \ge 1$. Cesa-Bianchi, Freund, Haussler, Helmbold, Schapire, and Warmuth~\cite{CBFHHSW97} (see also~\cite{CBL06, OP15}) showed that the leading constant $\frac{1}{\sqrt{2}}$ cannot be improved in the many-expert regime: for large $n$, no algorithm that knows $T$ can guarantee regret asymptotically smaller than $\sqrt{\frac{T \ln n}{2}}$.
% Throughout the introduction we suppress lower-order terms and write regret bounds informally.

A strictly more demanding requirement is that the regret bound must hold at \emph{every} time simultaneously, with no horizon supplied in advance. An algorithm meeting this requirement is called an \emph{anytime} guarantee: it uses no knowledge of the horizon, and its cumulative regret $R_t$ is bounded simultaneously for every $t \ge 1$. The natural question is how much this costs in the leading constant. For $n = 2$ experts the cost is real and quantified: Luo and Schapire~\cite{LS14} showed that any anytime algorithm suffers regret at least $\sqrt{t/\pi} \approx 0.5642\sqrt{t}$ at some time $t$, strictly more than Cover's~\cite{Cover65} fixed-horizon, asymptotic value of $\sqrt{t/(2\pi)} \approx 0.3989\sqrt{t}$, while Harvey, Liaw, Perkins, and Randhawa~\cite{HLPR20} determine the exact anytime constant, $\frac{\gamma}{2}\sqrt{t} \approx 0.6535\sqrt{t}$, where $\gamma \approx 1.3069$ is the smallest positive root of Kummer's confluent hypergeometric function $x \mapsto M(-1/2, 1/2, x^2/2)$. So for two experts, knowing the horizon genuinely helps.

In the many-expert regime ($n \to \infty$), an appropriately chosen decreasing learning rate uses no knowledge of the horizon and guarantees $R_t \le \sqrt{t \ln n}$ simultaneously at every $t \ge 1$~\cite{CZ10, Chernov10}. The only known obstruction is the lower bound $(1 - o(1))\sqrt{t \ln n / 2}$ inherited from the fixed-horizon setting~\cite{CBFHHSW97}; the two differ by exactly a factor of $\sqrt{2}$. Evidence that the lower end is the truth was obtained in continuous time by Harvey, Liaw, and Portella~\cite{HLP24}, \rev{who proved that when the $n$ experts follow independent Brownian motions, the anytime regret matches the fixed-time leading constant as $n \to \infty$.} Whether \rev{the factor of $\sqrt{2}$ can be removed} in the general \emph{adversarial discrete-time setting} has remained an open question.

\subsection{Our Contribution}

We show that the factor of $\sqrt{2}$ is not necessary. Our main result is the following theorem, proved in Section~\ref{sec:anytime_regret}.

\begin{theorem}\label{thm:main_intro}
\rev{There is an algorithm for prediction with expert advice which uses no knowledge of the time horizon whose regret over $n$ experts satisfies
\[
    R_t \le \left(1 + O\left(\sqrt{\frac{\ln \ln n}{\ln n}}\right)\right) \sqrt{\frac{t \ln n}{2}}
    \qquad \text{simultaneously for all } t \ge 1
\]
for every sequence of loss vectors $(\ell_t)_{t \ge 1} \in ([0,1]^n)^{\mathbb{N}}$.}
\end{theorem}

\rev{The leading constant is optimal. Even an algorithm that knows the horizon $T$ in advance cannot guarantee regret smaller than $(1 - o(1))\sqrt{\frac{T \ln n}{2}}$, as $T \to \infty$ and then $n \to \infty$~\cite{CBFHHSW97}, and an anytime algorithm is in particular a fixed-horizon algorithm at every $T$.}

\subsection{High-Level Proof Overview}

Consider an adversary who chooses an unknown time $t$ at which to evaluate the learner's cumulative regret. If we knew $t$ in advance---or could guess $t$ up to a multiplicative factor of $1 + \varepsilon$---then running standard Multiplicative Weights Update (MWU) with a fixed learning rate tuned for that horizon would achieve regret at most $\sqrt{1 + \varepsilon}\sqrt{\frac{t \ln n}{2}}$. Since $t$ is unknown, a natural strategy is to run a master no-regret algorithm whose ``experts'' are individual MWU instances running with different fixed-horizon learning rates on a geometric grid $H^m = (1 + \varepsilon)^m$.
% Similar ideas have been used in the literature: Luo and Schapire~\cite[Section~6.4]{LS14} average MWU predictors tuned to a range of candidate horizons, and running base learners over a geometric grid of time intervals goes back to Hazan and Seshadhri~\cite{HS09} and was developed further in the strongly adaptive regret literature~\cite{DGSS15, AKCV16, JOWW17}.
At any time $t$, at least one of these instances has horizon $H^m \in [t, (1+\varepsilon)t]$ and that instance incurs regret $\sqrt{1+\varepsilon}\sqrt{\frac{t \ln n}{2}}$.
If we can ensure the master algorithm incurs small regret relative to this MWU instance, say $o(\sqrt{t \ln n})$ regret then the overall regret at time $t$ would be
% as long as the master algorithm incurs small regret relative to that instance, our overall regret at time $t$ will also be
$(1 + o(1))\sqrt{\frac{t \ln n}{2}}$.

The difficulty is that the master's own regret may grow with the number of
instances it tracks, so a pool that grows with $t$ could leave a
time-dependent overhead in the leading constant.
We therefore consider a setup where we track a dynamic set of MWU instances.
Specifically, we wake up instance $m$ only at round $\delta H^m$, for some $\delta \in (0, 1)$ and then retire it permanently at round $H^m$.
We therefore wake instance $m$ only at round $\lfloor \delta H^m \rfloor$,
% for a fixed fraction $\delta \in (0,1)$, and retire it permanently after round
$H^m$.
This keeps $M = O(\frac{1}{\varepsilon} \ln \frac{1}{\delta})$ instances
awake at once, independent of $t$, and also limits to $O(\frac{1}{\varepsilon}
\ln \frac{1}{\delta})$ the number of instances created during any one
instance's lifetime---this last part is required to ensure that no expert gets diluted too much in its lifetime.

% the quantity that the aggregation cost really scales
% with, since each newcomer must be endowed with weight taken from the
% incumbents.

With this setup, we will be able to show that the regret with respect to instance $m$, in the interval where $m$ is alive, scales as $O(\sqrt{H^m \ln M}) = o(\sqrt{H^m \ln n})$ by choosing $\varepsilon, \delta > 0$ appropriately.
For example, we can take $\varepsilon = \delta \sim \frac{1}{\ln n}$ which would give a regret with respect to instance $m$ of $O(\sqrt{H^m \ln \ln n})$.
By summing up the regret with respect to instance $m$ and instance $m$'s own regret, we have that at time $t \in [H^{m-1}, H^m]$, the overall regret from time $\delta H^m$ to $t$ is about $(1 + o(1))\sqrt{t \ln n / 2}$.
We still need to obtain the regret from time $1$ through $\delta H^m \approx \delta t$ but this can be done by recursing to obtain total regret $(1 + o(1)) \sqrt{\frac{t\ln n}{2}} (1 + \sqrt{\delta} + \sqrt{\delta}^2 + \ldots)$.
The latter infinite sum converges to $\frac{1}{1-\sqrt{\delta}} \approx 1 + \sqrt{\delta} = 1 + o(1)$, allowing us to conclude that the overall regret is $(1 + o(1)) \sqrt{\frac{t\ln n}{2}}$.

% \rev{Retiring instances, however, means that the one relevant at time $t$ has
% not been running from the start: it woke only at round
% $\lfloor \delta H^m \rfloor$, so its guarantee says nothing about the rounds
% before that. Over the stretch where it is awake its own
% regret is at most $\sqrt{\frac{H^m \ln n}{2}} \le \sqrt{1 + \varepsilon}
% \sqrt{\frac{t \ln n}{2}}$, and the master pays only $O(\sqrt{t \ln \ln n})$
% relative to it. That leaves the prefix $[1, \lfloor \delta H^m \rfloor)$, whose
% length is at most $\delta(1 + \varepsilon)t$, and on which we simply recurse.
% Taking $\varepsilon = \sqrt{\ln \ln n / \ln n}$ and $\delta = \varepsilon^3$, the
% grid coarseness, the master's overhead, and the recursion together cost a factor
% $1 + O(\sqrt{\ln \ln n / \ln n})$.}

\subsection{Related Work}\label{sec:related_work}

\emph{The anytime setting.} Anytime guarantees are usually obtained through a doubling trick, through a decreasing learning rate~\cite{ACG02, CBL06, Hazan16}, or through self-tuning methods that carry no horizon parameter at all~\cite{CFH09, dRvEGK14, KvE15, LS15, OP16, FHPQW26}. Far less is known about the optimal constant. For $n = 2$, Luo and Schapire~\cite{LS14} showed that any anytime algorithm can be forced to incur regret $\sqrt{t/\pi}$, strictly more than the fixed-time value $\sqrt{t/(2\pi)}$~\cite{Cover65}, and Harvey, Liaw, Perkins, and Randhawa~\cite{HLPR20} determined the exact constant $\gamma/2 \approx 0.653$, obtaining a matching anytime algorithm by solving a continuous analogue of the regret problem with stochastic calculus; the constant $\gamma$ itself comes from the study of Brownian motion at square-root boundaries~\cite{Breiman67, Davis76, GP83}. For large $n$, no lower bound is known that exceeds the fixed-horizon value $\sqrt{t \ln n / 2}$. The only prior evidence about the constant is in continuous time: Harvey, Liaw, and Portella~\cite{HLP24} showed that when the experts follow independent Brownian motions, the anytime and fixed-time constants agree as $n \to \infty$. 
% Theorem~\ref{thm:main_intro} is the adversarial discrete-time analogue.
A closely related probabilistic statement appears in~\cite{HLP25}, which in our normalization---per-coordinate increments bounded in absolute value by $\frac{1}{2}$---says that $\mathbb{E}\|S_\tau\|_\infty \le \bigl(\frac{3}{2} + \sqrt{\frac{\ln n}{2}}\bigr)\,\mathbb{E}\sqrt{\tau}$ for discrete martingales, whenever $S_0 = 0$ and $\mathbb{E} \sqrt{\tau} < \infty$.
% again with the fixed-time leading constant.
% This martingale inequality does not directly give the sharp adversarial regret guarantee proved here.

\emph{\rev{Exact constants for small $n$.}} \rev{In the fixed-time and geometric-stopping settings the exact constants are known for $n \le 4$, and for $n = 5$ under geometric stopping, through a line of work connecting the experts problem to nonlinear partial differential equations~\cite{GPS16, ABG17, DK20, KKW20, BEYZ20, BEXZ20, CD26, BEK26} (see also~\cite{GHP22}). Complementary computational investigations of COMB strategies and the associated PDEs appear in~\cite{Chase19, CDM26}.}

\emph{\rev{Techniques.}} \rev{We use a number of techniques that have appeared in the literature: power rescaling, which lets multiplicative updates coexist with a decreasing learning rate~\cite{ACG02, CBGLS12}; Fixed Share and generalized sharing~\cite{HW98, BW02}; regret with experts arriving over time~\cite{MM17}; specialist and sleeping-expert reductions~\cite{FSSW97, BM07, KNMS10, GSvE14}; and geometric interval grids for adaptive and strongly adaptive learning, whose base learners are started and retired at prescribed rounds so as to keep the active pool small~\cite{HS09, DGSS15, AKCV16, JOWW17, WTZ22}. Aggregation over a grid of learning rates rather than of time intervals plays a similar role in~\cite{KvE15, vEK16, vEK21}; for example, \cite{vEK21} also starts and stops its learning-rate experts on the fly, keeping $O(\ln T)$ of them active at once. Mixtures over horizon-tuned MWU instances appear already in Luo and Schapire~\cite[Section~6.4]{LS14}, with the weights over horizons fixed in advance rather than learned. What is different here is that the active horizons must lengthen with $t$: our construction attains the leading constant at time $t$ by keeping awake an instance whose horizon lies in $[t, (1 + \varepsilon)t]$.}

\section{Preliminaries and Problem Formulation}\label{sec:prelims}

We consider sequential prediction with expert advice over $n \ge 2$ experts. Let $[n] := \{1, 2, \dots, n\}$ denote the set of experts, and let
\[
    \Delta_n := \left\{ p = (p_1, \dots, p_n) \in \mathbb{R}_{\ge 0}^n : \sum_{i=1}^n p_i = 1 \right\}
\]
denote the probability simplex over $[n]$ (and, more generally, for any finite set $S$, let $\Delta(S) := \{ p \in \mathbb{R}_{\ge 0}^S : \sum_{j \in S} p_j = 1 \}$ denote the probability simplex over $S$; we write $A \sqcup B$ for the union of disjoint sets $A$ and $B$). The sequential interaction protocol is summarized in Algorithm~\ref{alg:pea_protocol}.

\begin{algorithm}[ht]
\caption{Online Prediction with Expert Advice}
\label{alg:pea_protocol}
\vspace{3pt}
\noindent\textbf{Parameters:} Number of experts $n \ge 2$; expert set $[n] = \{1, 2, \dots, n\}$.\\[3pt]
\textbf{for} each round $t = 1, 2, \dots$ \textbf{do}
\begin{enumerate}
    \setlength{\itemsep}{2pt}
    \setlength{\parskip}{0pt}
    \setlength{\topsep}{3pt}
    \item The learner selects a distribution $p_t = (p_{t,1}, \dots, p_{t,n}) \in \Delta_n$ based on past losses $(\ell_1, \dots, \ell_{t-1})$.
    \item The adversary reveals a loss vector $\ell_t = (\ell_{t,1}, \dots, \ell_{t,n}) \in [0, 1]^n$.
    \item The learner incurs expected instantaneous loss $\langle p_t, \ell_t \rangle = \sum_{i=1}^n p_{t,i} \ell_{t,i}$.
\end{enumerate}
\textbf{end for}
\vspace{2pt}
\end{algorithm}

For any time horizon $T \ge 1$, the cumulative regret of the learner with respect to expert $i \in [n]$ is
\[
    R_{T,i} := \sum_{t=1}^T \langle p_t, \ell_t \rangle - \sum_{t=1}^T \ell_{t,i} = \sum_{t=1}^T \left( \langle p_t, \ell_t \rangle - \ell_{t,i} \right),
\]
and the worst-case cumulative regret across all $n$ experts at horizon $T$ is
\[
    R_T := \max_{i \in [n]} R_{T,i} = \sum_{t=1}^T \langle p_t, \ell_t \rangle - \min_{i \in [n]} \sum_{t=1}^T \ell_{t,i}.
\]
An online learning algorithm is said to guarantee an \emph{anytime regret bound} $B(n, T)$ if it operates without prior knowledge of the time horizon $T$ and satisfies $R_T \le B(n, T)$ simultaneously for all $T \ge 1$ and all loss sequences $(\ell_t)_{t \ge 1} \in ([0, 1]^n)^{\mathbb{N}}$.

\section{Online Prediction With Entering and Exiting Experts}\label{sec:dynamic_mwu}

To aggregate base learners across an evolving grid of horizons without accumulating an unbounded active pool, we first study sequential prediction with expert advice where the set of active experts changes over time~\cite{ACG02, BW02, CBGLS12, CBL06, HW98, MM17}.

At each discrete round $t = 1, 2, \dots$, there is a finite non-empty set $S_t$ of active experts, satisfying $|S_t| \le M$ for all $t \ge 1$. At round $t$, the learner plays a distribution $p_t \in \Delta(S_t)$, observes a loss vector $\ell_t \in [0, 1]^{S_t}$, and incurs expected loss $\langle p_t, \ell_t \rangle = \sum_{j \in S_t} p_{t,j} \ell_{t,j}$. Between round $t$ and round $t+1$, the active set transitions from $S_t$ to $S_{t+1}$, which partitions as
\[
    S_{t+1} = S_{\mathrm{Surv}, t} \sqcup S_{\mathrm{Enter}, t+1},
\]
where $S_{\mathrm{Surv}, t} := S_t \cap S_{t+1}$ is the set of surviving experts (assumed non-empty) and $S_{\mathrm{Enter}, t+1} := S_{t+1} \setminus S_t$ is the set of newly entering experts, of size $k_t := |S_{\mathrm{Enter}, t+1}|$. Each entering expert is allocated an initial weight $a \in (0, 1/M]$, and we assume $a k_t \le c < 1$ for all $t \ge 1$.

Algorithm~\ref{alg:dynamic_mwu} gives the update rule for a non-increasing sequence of learning rates $(\eta_t)_{t \ge 1}$ ($\eta_1 \ge \eta_2 \ge \dots > 0$).

\begin{algorithm}[ht]
\caption{Multiplicative Weights with Entering and Exiting Experts}
\label{alg:dynamic_mwu}
\vspace{3pt}
\noindent\textbf{Input:} Non-increasing learning rates $\eta_1 \ge \eta_2 \ge \dots > 0$; entry weight $a \in (0, 1/M]$; initial active set $S_1$.\\[2pt]
\noindent\textbf{Initialization:} Set $p_1 \in \Delta(S_1)$ uniformly, i.e., $p_{1,j} := 1/|S_1|$ for all $j \in S_1$.\\[3pt]
\textbf{for} each round $t = 1, 2, \dots$ \textbf{do}
\begin{enumerate}
    \setlength{\itemsep}{3pt}
    \setlength{\parskip}{0pt}
    \setlength{\topsep}{3pt}
    \item \emph{Prediction and loss:} Play $p_t \in \Delta(S_t)$, observe $\ell_t \in [0, 1]^{S_t}$, and incur loss $\langle p_t, \ell_t \rangle = \sum_{j \in S_t} p_{t,j} \ell_{t,j}$.
    \item \emph{Multiplicative weights update:} For each $j \in S_t$, compute
    \begin{equation}\label{eq:standard_mwu_update}
        p'_{t,j} := \frac{p_{t,j} \exp(-\eta_t \ell_{t,j})}{Z_t},
        \qquad Z_t := \sum_{r \in S_t} p_{t,r} \exp(-\eta_t \ell_{t,r}) > 0.
    \end{equation}
    \item \emph{Expert exit and entry:} Given $S_{t+1} = S_{\mathrm{Surv}, t} \sqcup S_{\mathrm{Enter}, t+1}$ with $k_t = |S_{\mathrm{Enter}, t+1}|$, set
    \begin{equation}\label{eq:exit_renorm}
        \hat{p}'_{t,j} := \frac{p'_{t,j}}{\sum_{r \in S_{\mathrm{Surv}, t}} p'_{t,r}} \quad (j \in S_{\mathrm{Surv}, t}), \qquad
        \tilde{p}_{t,j} := \begin{cases}
            (1 - a k_t) \, \hat{p}'_{t,j} & \text{if } j \in S_{\mathrm{Surv}, t}, \\
            a & \text{if } j \in S_{\mathrm{Enter}, t+1}.
        \end{cases}
    \end{equation}
    \item \emph{Learning-rate rescaling:} Set $p_{t+1} \in \Delta(S_{t+1})$ by
    \begin{equation}\label{eq:power_rescaling}
        p_{t+1,j} := \frac{(\tilde{p}_{t,j})^{\eta_{t+1} / \eta_t}}{\sum_{r \in S_{t+1}} (\tilde{p}_{t,r})^{\eta_{t+1} / \eta_t}} \quad (j \in S_{t+1}).
    \end{equation}
\end{enumerate}
\textbf{end for}
\vspace{2pt}
\end{algorithm}

\begin{theorem}\label{thm:dynamic_mwu}
Let $(\eta_t)_{t \ge 1}$ be a non-increasing sequence of learning rates ($\eta_1 \ge \eta_2 \ge \dots > 0$), and assume $|S_t| \le M$ and $a k_t \le c < 1$ for all $t \ge 1$, where $a \in (0, 1/M]$. For any horizon $T \ge 1$ and any interval $[t_0, T] \subseteq \{1, \dots, T\}$ in which expert $i$ enters at round $t_0$ and remains active through round $T$ ($i \in S_{\mathrm{Surv}, t}$ for all $t \in \{t_0, \dots, T\}$, setting $S_{T+1} := S_T$ so $k_T = 0$), if $\sum_{t=t_0}^{T-1} k_t \le K$, then
\begin{equation}\label{eq:cumulative_mwu_bound}
    \sum_{t=t_0}^T \left( \langle p_t, \ell_t \rangle - \ell_{t,i} \right)
    \le \frac{1}{\eta_{t_0}} \ln \frac{1}{a} + \left( \frac{1}{\eta_T} - \frac{1}{\eta_{t_0}} \right) \ln M + \frac{a K}{(1 - c) \eta_T} + \frac{1}{8} \sum_{t=t_0}^T \eta_t.
\end{equation}
\end{theorem}

To prove Theorem~\ref{thm:dynamic_mwu}, we first establish a one-step bound for surviving experts. We recall Hoeffding's moment generating function inequality for bounded random variables~\cite[Eq.~(4.16)]{Hoeffding63}: if $X$ is a real-valued random variable taking values in $[u, v]$ ($u \le v$) almost surely, then for every $\lambda \in \mathbb{R}$,
\[
    \ln \mathbb{E}\left[ e^{\lambda(X - \mathbb{E}[X])} \right] \le \frac{\lambda^2 (v-u)^2}{8}.
\]

\begin{lemma}\label{lem:one_step_mwu}
Fix any round $t \ge 1$, learning rate $\eta_t > 0$, active sets $S_t$ and $S_{t+1}$, distribution $p_t \in \Delta(S_t)$, and loss vector $\ell_t \in [0, 1]^{S_t}$. Let $\tilde{p}_t \in \Delta(S_{t+1})$ be defined by~\eqref{eq:exit_renorm}, where we choose $a > 0$ so that $1 - a k_t > 0$. Then for every surviving expert $i \in S_{\mathrm{Surv}, t} = S_t \cap S_{t+1}$ with $p_{t,i} > 0$, we have $\tilde{p}_{t,i} > 0$ and
\begin{equation}\label{eq:one_step_enter_exit}
    \langle p_t, \ell_t \rangle - \ell_{t,i}
    \le \frac{1}{\eta_t} \left( \ln \frac{1}{p_{t,i}} - \ln \frac{1}{\tilde{p}_{t,i}} \right) + \frac{1}{\eta_t} \ln \frac{1}{1 - a k_t} + \frac{\eta_t}{8}.
\end{equation}
\end{lemma}

\begin{proof}
Since $p_t \in \Delta(S_t)$ satisfies $\sum_{k \in S_t} p_{t,k} = 1$ and $\exp(-\eta_t \ell_{t,k}) > 0$ for all $k \in S_t$, the normalization constant $Z_t = \sum_{k \in S_t} p_{t,k} \exp(-\eta_t \ell_{t,k})$ is strictly positive. Hence, for any expert $i \in S_t$ with $p_{t,i} > 0$, we have $p'_{t,i} = \frac{p_{t,i} \exp(-\eta_t \ell_{t,i})}{Z_t} > 0$, so $\ln \frac{1}{p_{t,i}}$ and $\ln \frac{1}{p'_{t,i}}$ are finite and well-defined.

Let $J \sim p_t$ be a random expert index drawn according to $p_t$, and consider the random variable $X = \ell_{t,J}$. Since $\ell_{t,j} \in [0, 1]$ for all $j \in S_t$, $X$ takes values in $[0, 1]$ with expectation $\mathbb{E}[X] = \langle p_t, \ell_t \rangle$. By definition of $Z_t$,
\[
    Z_t = \mathbb{E}\left[ \exp(-\eta_t X) \right] = \exp(-\eta_t \langle p_t, \ell_t \rangle) \, \mathbb{E}\left[ \exp\bigl(-\eta_t (X - \mathbb{E}[X])\bigr) \right].
\]
Taking natural logarithms on both sides and applying Hoeffding's moment generating function inequality with $\lambda = -\eta_t$ and range $v - u = 1 - 0 = 1$ yields
\[
    \ln Z_t = -\eta_t \langle p_t, \ell_t \rangle + \ln \mathbb{E}\left[ \exp\bigl(-\eta_t (X - \mathbb{E}[X])\bigr) \right] \le -\eta_t \langle p_t, \ell_t \rangle + \frac{\eta_t^2}{8}.
\]
Multiplying by $-1$ gives
\begin{equation}\label{eq:log_Zt_lower}
    -\ln Z_t \ge \eta_t \langle p_t, \ell_t \rangle - \frac{\eta_t^2}{8}.
\end{equation}
On the other hand, taking natural logarithms of~\eqref{eq:standard_mwu_update} for expert $i \in S_t$ gives
\[
    \ln p'_{t,i} = \ln p_{t,i} - \eta_t \ell_{t,i} - \ln Z_t,
\]
which rearranges to
\begin{equation}\label{eq:log_Zt_exact}
    -\ln Z_t = \eta_t \ell_{t,i} + \ln \frac{p'_{t,i}}{p_{t,i}} = \eta_t \ell_{t,i} + \left( \ln \frac{1}{p_{t,i}} - \ln \frac{1}{p'_{t,i}} \right).
\end{equation}
Equating the right-hand side of~\eqref{eq:log_Zt_exact} with the lower bound~\eqref{eq:log_Zt_lower} and dividing by $\eta_t > 0$ yields
\begin{equation}\label{eq:one_step_standard}
    \langle p_t, \ell_t \rangle - \ell_{t,i}
    \le \frac{1}{\eta_t} \left( \ln \frac{1}{p_{t,i}} - \ln \frac{1}{p'_{t,i}} \right) + \frac{\eta_t}{8}.
\end{equation}
Finally, consider any surviving expert $i \in S_{\mathrm{Surv}, t} = S_t \cap S_{t+1}$. Since $p'_{t,i} > 0$, the surviving probability mass satisfies $\sum_{r \in S_{\mathrm{Surv}, t}} p'_{t,r} \ge p'_{t,i} > 0$, so $\hat{p}'_t$ in~\eqref{eq:exit_renorm} is well-defined and satisfies
\[
    \hat{p}'_{t,i}
    = \frac{p'_{t,i}}{\sum_{r \in S_{\mathrm{Surv}, t}} p'_{t,r}}
    \ge \frac{p'_{t,i}}{\sum_{r \in S_t} p'_{t,r}}
    = p'_{t,i},
\]
because $\sum_{r \in S_{\mathrm{Surv}, t}} p'_{t,r} \le \sum_{r \in S_t} p'_{t,r} = 1$. Since $1 - a k_t > 0$, we have $\tilde{p}_{t,i} = (1 - a k_t) \hat{p}'_{t,i} \ge (1 - a k_t) p'_{t,i} > 0$. Taking natural logarithms and multiplying by $-1$ gives
\[
    \ln \frac{1}{\tilde{p}_{t,i}}
    \le \ln \frac{1}{p'_{t,i}} + \ln \frac{1}{1 - a k_t},
\]
or equivalently, $-\ln \frac{1}{p'_{t,i}} \le -\ln \frac{1}{\tilde{p}_{t,i}} + \ln \frac{1}{1 - a k_t}$. Substituting this inequality into~\eqref{eq:one_step_standard} establishes~\eqref{eq:one_step_enter_exit}.
\end{proof}

\begin{proof}[Proof of Theorem~\ref{thm:dynamic_mwu}]
\rev{We first record that all weights remain strictly positive, which is required to apply Lemma~\ref{lem:one_step_mwu} at every round. Indeed, $p_{1,j} = 1/|S_1| > 0$ for $j \in S_1$; and if $p_{t,j} > 0$ for all $j \in S_t$, then $p'_{t,j} > 0$ by~\eqref{eq:standard_mwu_update} (since $\exp(-\eta_t \ell_{t,j}) > 0$ and $Z_t > 0$), hence $\hat{p}'_{t,j} > 0$ for $j \in S_{\mathrm{Surv},t}$, hence $\tilde{p}_{t,j} > 0$ for all $j \in S_{t+1}$ (using $1 - a k_t \ge 1 - c > 0$ for survivors and $a > 0$ for entrants), and therefore $p_{t+1,j} > 0$ for all $j \in S_{t+1}$ by~\eqref{eq:power_rescaling}.} Fix any round $t \ge 1$ and any expert $i \in S_{t+1}$ with $\tilde{p}_{t,i} > 0$. Taking logarithms in~\eqref{eq:power_rescaling} gives
\[
    \ln \frac{1}{p_{t+1,i}}
    = \frac{\eta_{t+1}}{\eta_t} \ln \frac{1}{\tilde{p}_{t,i}} + \ln \left( \sum_{r \in S_{t+1}} (\tilde{p}_{t,r})^{\eta_{t+1} / \eta_t} \right).
\]
Dividing both sides by $\eta_{t+1} > 0$ and subtracting $\frac{1}{\eta_t} \ln \frac{1}{\tilde{p}_{t,i}}$ yields
\[
    \frac{1}{\eta_{t+1}} \ln \frac{1}{p_{t+1,i}} - \frac{1}{\eta_t} \ln \frac{1}{\tilde{p}_{t,i}}
    = \frac{1}{\eta_{t+1}} \ln \left( \sum_{r \in S_{t+1}} (\tilde{p}_{t,r})^{\eta_{t+1} / \eta_t} \right).
\]
Let $\rho_t := \frac{\eta_{t+1}}{\eta_t}$. Since $0 < \eta_{t+1} \le \eta_t$, we have $\rho_t \in (0, 1]$, so the function $x \mapsto x^{\rho_t}$ is concave on $[0, \infty)$. By Jensen's inequality applied to the average over the $|S_{t+1}|$ coordinates $r \in S_{t+1}$,
\[
    \frac{1}{|S_{t+1}|} \sum_{r \in S_{t+1}} (\tilde{p}_{t,r})^{\rho_t}
    \le \left( \frac{1}{|S_{t+1}|} \sum_{r \in S_{t+1}} \tilde{p}_{t,r} \right)^{\rho_t}
    = \left( \frac{1}{|S_{t+1}|} \right)^{\rho_t}
    = |S_{t+1}|^{-\rho_t},
\]
where we used $\sum_{r \in S_{t+1}} \tilde{p}_{t,r} = 1$. Multiplying by $|S_{t+1}|$ and using $|S_{t+1}| \le M$ gives
\[
    \sum_{r \in S_{t+1}} (\tilde{p}_{t,r})^{\rho_t}
    \le |S_{t+1}|^{1 - \rho_t}
    \le M^{1 - \rho_t}.
\]
Taking logarithms and dividing by $\eta_{t+1} > 0$ yields
\begin{equation}\label{eq:rescaling_step_bound}
    \frac{1}{\eta_{t+1}} \ln \frac{1}{p_{t+1,i}} - \frac{1}{\eta_t} \ln \frac{1}{\tilde{p}_{t,i}}
    \le \left( \frac{1}{\eta_{t+1}} - \frac{1}{\eta_t} \right) \ln M.
\end{equation}
Next, consider any interval $[t_0, T]$ in which expert $i$ enters at round $t_0$ and remains active through round $T$. If $t_0 = 1$, then $p_{1,i} \ge a$ implies $\frac{1}{\eta_1} \ln \frac{1}{p_{1,i}} \le \frac{1}{\eta_1} \ln \frac{1}{a}$. If $t_0 > 1$, then expert $i \in S_{\mathrm{Enter}, t_0}$ is assigned $\tilde{p}_{t_0-1,i} = a \le 1/M$ in~\eqref{eq:exit_renorm}, and applying~\eqref{eq:rescaling_step_bound} at round $t_0 - 1$ with $\ln M \le \ln \frac{1}{a}$ and $\frac{1}{\eta_{t_0}} \ge \frac{1}{\eta_{t_0-1}}$ yields
\[
    \frac{1}{\eta_{t_0}} \ln \frac{1}{p_{t_0,i}}
    \le \frac{1}{\eta_{t_0-1}} \ln \frac{1}{a} + \left( \frac{1}{\eta_{t_0}} - \frac{1}{\eta_{t_0-1}} \right) \ln M
    \le \frac{1}{\eta_{t_0}} \ln \frac{1}{a}.
\]
Summing Lemma~\ref{lem:one_step_mwu} over $t = t_0, \dots, T$ (with $k_T = 0$) and applying~\eqref{eq:rescaling_step_bound} at rounds $t = t_0, \dots, T-1$ telescopes the potential terms:
\begin{align*}
    \sum_{t=t_0}^T \frac{1}{\eta_t} \left( \ln \frac{1}{p_{t,i}} - \ln \frac{1}{\tilde{p}_{t,i}} \right)
    &= \frac{1}{\eta_{t_0}} \ln \frac{1}{p_{t_0,i}}
    + \sum_{t=t_0}^{T-1} \left( \frac{1}{\eta_{t+1}} \ln \frac{1}{p_{t+1,i}} - \frac{1}{\eta_t} \ln \frac{1}{\tilde{p}_{t,i}} \right)
    - \frac{1}{\eta_T} \ln \frac{1}{\tilde{p}_{T,i}} \\
    &\le \frac{1}{\eta_{t_0}} \ln \frac{1}{a}
    + \sum_{t=t_0}^{T-1} \left( \frac{1}{\eta_{t+1}} - \frac{1}{\eta_t} \right) \ln M \\
    &= \frac{1}{\eta_{t_0}} \ln \frac{1}{a} + \left( \frac{1}{\eta_T} - \frac{1}{\eta_{t_0}} \right) \ln M,
\end{align*}
where we dropped the non-positive boundary term $-\frac{1}{\eta_T} \ln \frac{1}{\tilde{p}_{T,i}} \le 0$ (since $\tilde{p}_{T,i} \in (0, 1]$). Finally, for any $t \in \{t_0, \dots, T-1\}$, since $0 \le a k_t \le c < 1$, we have
\[
    \ln \frac{1}{1 - a k_t}
    = \int_0^{a k_t} \frac{du}{1 - u}
    \le \frac{a k_t}{1 - c},
\]
and using $\frac{1}{\eta_t} \le \frac{1}{\eta_T}$ (as $\eta_t \ge \eta_T$) together with $\sum_{t=t_0}^{T-1} k_t \le K$ gives $\sum_{t=t_0}^{T-1} \frac{1}{\eta_t} \ln \frac{1}{1 - a k_t} \le \frac{a K}{(1 - c) \eta_T}$, establishing~\eqref{eq:cumulative_mwu_bound}.
\end{proof}

\section{Asymptotically Optimal Anytime Regret}\label{sec:anytime_regret}

We now construct our anytime algorithm by running Algorithm~\ref{alg:dynamic_mwu} over a geometric schedule of base learners. Fix parameters $\varepsilon \in (0, 1)$ and $\delta \in (0, 1)$ satisfying $\delta (1 + \varepsilon) < 1$, and set $H := 1 + \varepsilon$. For each integer $m \ge 1$, let expert $m$ wake up at round $t_{\mathrm{start}}(m) := \max\{ 1, \lfloor \delta H^m \rfloor \}$ and retire after round $t_{\mathrm{end}}(m) := \lfloor H^m \rfloor$, so that expert $m$ is active during the interval $I_m := [t_{\mathrm{start}}(m), t_{\mathrm{end}}(m)]$. Upon waking up at round $t_{\mathrm{start}}(m)$, expert $m$ runs standard Multiplicative Weights over the $n$ base experts $[n]$ starting from the uniform distribution $(1/n, \dots, 1/n)$ with fixed learning rate $\gamma_m := \sqrt{8 \ln n / H^m}$. Let $q^{(m)}_t \in \Delta_n$ denote the distribution played by expert $m$ at round $t \in I_m$. At each round $t \ge 1$, our master algorithm runs Algorithm~\ref{alg:dynamic_mwu} over the active set $S_t := \{ m \ge 1 : t \in I_m \}$ with learning rate $\eta_t := \sqrt{\frac{\ln M}{t}}$ and entry weight $a := \frac{1}{4 M}$, where $M := \left\lceil \frac{\ln(2/\delta)}{\ln(1 + \varepsilon)} \right\rceil + 1$, and plays the aggregated distribution
\[
    p_t := \sum_{m \in S_t} w_{t,m} \, q^{(m)}_t \in \Delta_n,
\]
where $w_t \in \Delta(S_t)$ is the weight vector maintained by Algorithm~\ref{alg:dynamic_mwu} (initialized at round $t=1$ as the uniform distribution $w_{1,m} = 1/|S_1| \ge 1/M > a$ for $m \in S_1$).

In the remainder of this section, we prove that this algorithm enjoys good regret.
Fix a time $t$.
First, Lemma~\ref{lem:base_expert_regret} shows that all the MWU instances have good regret; this is standard and well-known.
Second, we need to show that there is an MWU instance $m$ which has been active for almost the entire period, specifically active since about time $\delta t$ (Lemma~\ref{lem:active_existence}).
Third, we need to show that the number of active MWU instances at any time is small (Lemma~\ref{lem:pool_size_M}) and that a limited number of MWU instances arrive \emph{after} MWU instance $m$ (Lemma~\ref{lem:arrival_bound_K}).
This allows us to apply Theorem~\ref{thm:dynamic_mwu}.
Lemma~\ref{lem:epoch_regret} then puts these pieces together to get a regret bound for the entire algorithm from time about $\delta t$ to $t$ while Lemma~\ref{lem:main_anytime} uses induction to obtain low regret for the entire time interval starting at time $1$.
Finally, we prove Theorem~\ref{thm:main_intro} by instantiating Lemma~\ref{lem:main_anytime} with the appropriate parameters.

\begin{lemma}\label{lem:base_expert_regret}
For every $m \ge 1$, every round $t \in I_m$, and every base expert $j \in [n]$,
\[
    \sum_{s = t_{\mathrm{start}}(m)}^t \left( \langle q^{(m)}_s, \ell_s \rangle - \ell_{s, j} \right)
    \le \sqrt{\frac{H^m \ln n}{2}}.
\]
\end{lemma}

\begin{proof}
Applying~\eqref{eq:one_step_standard} to expert $m$ over rounds $s \in [t_{\mathrm{start}}(m), t]$ with initial distribution $q^{(m)}_{t_{\mathrm{start}}(m), j} = 1/n$ and fixed learning rate $\gamma_m = \sqrt{\frac{8 \ln n}{H^m}}$ gives
\[
    \sum_{s = t_{\mathrm{start}}(m)}^t \left( \langle q^{(m)}_s, \ell_s \rangle - \ell_{s, j} \right)
    \le \frac{\ln n}{\gamma_m} + \frac{\gamma_m (t - t_{\mathrm{start}}(m) + 1)}{8}
    \le \frac{\ln n}{\gamma_m} + \frac{\gamma_m H^m}{8}
    = \sqrt{\frac{H^m \ln n}{2}},
\]
where we used $t - t_{\mathrm{start}}(m) + 1 \le \lfloor H^m \rfloor \le H^m$.
\end{proof}

\begin{lemma}\label{lem:active_existence}
For every round $t \ge 1$, there exists an integer $m \ge 1$ with $H^{m-1} \le t \le H^m$ (with $m=1$ if $t=1$) such that $t \in I_m$, and
\[
    t_{\mathrm{start}}(m) - 1 \le \delta (1 + \varepsilon) t.
\]
\end{lemma}

\begin{proof}
Since $H > 1$, the intervals $[1, H]$ and $(H^{m-1}, H^m]$ ($m \ge 2$) partition $[1, \infty)$. If $t \in (H^{m-1}, H^m]$ (or $t \in [1, H]$ with $m=1$), then $t \le \lfloor H^m \rfloor = t_{\mathrm{end}}(m)$. Furthermore, since $\delta H < 1$, we have $\lfloor \delta H^m \rfloor \le \delta H^m = (\delta H) H^{m-1} < H^{m-1} \le t$, so $t_{\mathrm{start}}(m) = \max\{1, \lfloor \delta H^m \rfloor\} \le t$. Thus $t \in I_m$, and $t_{\mathrm{start}}(m) - 1 = \max\{0, \lfloor \delta H^m \rfloor - 1\} \le \delta H^m \le \delta H t = \delta (1 + \varepsilon) t$.
\end{proof}

\begin{lemma}\label{lem:pool_size_M}
For every round $t \ge 1$,
\[
    |S_t| \le M := \left\lceil \frac{\ln(2/\delta)}{\ln(1 + \varepsilon)} \right\rceil + 1.
\]
\end{lemma}

\begin{proof}
If $m \in S_t$, then $t_{\mathrm{start}}(m) \le t \le t_{\mathrm{end}}(m)$, so $t \le H^m$ and $\lfloor \delta H^m \rfloor \le t$. The latter inequality implies $\delta H^m < t + 1 \le 2 t$, so $H^m < \frac{2 t}{\delta}$. Taking logarithms base $H$ gives
\[
    \frac{\ln t}{\ln H} \le m < \frac{\ln t + \ln(2/\delta)}{\ln H}.
\]
The number of integers $m$ in this half-open interval of length $\frac{\ln(2/\delta)}{\ln(1 + \varepsilon)}$ is at most $\left\lceil \frac{\ln(2/\delta)}{\ln(1 + \varepsilon)} \right\rceil + 1 = M$.
\end{proof}

\begin{lemma}\label{lem:arrival_bound_K}
For every $m \ge 1$, the number of experts entering during $I_m$ satisfies
\[
    \sum_{s = t_{\mathrm{start}}(m)}^{t_{\mathrm{end}}(m) - 1} k_s \le K := \left\lceil \frac{\ln(2/\delta)}{\ln(1 + \varepsilon)} \right\rceil + 1.
\]
\end{lemma}

\begin{proof}
Any expert $r$ that enters at some round $s+1 \in [t_{\mathrm{start}}(m)+1, t_{\mathrm{end}}(m)]$ must satisfy $r > m$ and $t_{\mathrm{start}}(r) = \lfloor \delta H^r \rfloor \le t_{\mathrm{end}}(m) \le H^m$ (since $t_{\mathrm{start}}(r) = s+1 \ge 2$), which implies $\delta H^r < H^m + 1 \le 2 H^m$, or equivalently $H^{r - m} < \frac{2}{\delta}$. Thus $m < r < m + \frac{\ln(2/\delta)}{\ln(1 + \varepsilon)}$, so at most $K = \left\lceil \frac{\ln(2/\delta)}{\ln(1 + \varepsilon)} \right\rceil + 1 = M$ experts enter during $I_m$. \rev{Moreover, the per-round arrival count also satisfies $k_s \le M$: since $S_{\mathrm{Enter}, s+1} \subseteq S_{s+1}$, Lemma~\ref{lem:pool_size_M} gives $k_s \le |S_{s+1}| \le M = K$.} Hence, since $a = \frac{1}{4 M}$, we have $a \le \frac{1}{M}$ and $a k_s \le a K = \frac{1}{4} =: c < 1$ for all $s$.
\end{proof}

\begin{lemma}\label{lem:epoch_regret}
For any round $t \ge 1$, let $m \ge 1$ be the expert from Lemma~\ref{lem:active_existence} with $H^{m-1} \le t \le H^m$, and let $t_0 := t_{\mathrm{start}}(m)$. Then for every base expert $j \in [n]$,
\begin{equation}\label{eq:epoch_regret_bound}
    \sum_{s = t_0}^t \left( \langle p_s, \ell_s \rangle - \ell_{s, j} \right)
    \le \sqrt{1 + \varepsilon} \sqrt{\frac{t \ln n}{2}} + C_{\varepsilon, \delta} \sqrt{t},
\end{equation}
where $C_{\varepsilon, \delta} := 3 \sqrt{\ln \left( \frac{4 \ln(2/\delta)}{\varepsilon} \right)}$.
\end{lemma}

\begin{proof}
For each round $s \in [t_0, t]$ and each active expert $r \in S_s$, let $\tilde{\ell}_{s, r} := \langle q^{(r)}_s, \ell_s \rangle \in [0, 1]$ denote the loss incurred by expert $r$ at round $s$. Since $p_s = \sum_{r \in S_s} w_{s, r} q^{(r)}_s$, linearity of the inner product gives
\[
    \langle p_s, \ell_s \rangle
    = \sum_{r \in S_s} w_{s, r} \langle q^{(r)}_s, \ell_s \rangle
    = \langle w_s, \tilde{\ell}_s \rangle.
\]
Adding and subtracting $\tilde{\ell}_{s, m} = \langle q^{(m)}_s, \ell_s \rangle$ decomposes the learner's instantaneous regret against base expert $j \in [n]$ into the master algorithm's regret against expert $m$ plus expert $m$'s regret against base expert $j$:
\[
    \langle p_s, \ell_s \rangle - \ell_{s, j}
    = \bigl( \langle p_s, \ell_s \rangle - \langle q^{(m)}_s, \ell_s \rangle \bigr)
    + \bigl( \langle q^{(m)}_s, \ell_s \rangle - \ell_{s, j} \bigr)
    = \bigl( \langle w_s, \tilde{\ell}_s \rangle - \tilde{\ell}_{s, m} \bigr)
    + \bigl( \langle q^{(m)}_s, \ell_s \rangle - \ell_{s, j} \bigr).
\]
Summing this identity over $s = t_0, \dots, t$ yields
\[
    \sum_{s = t_0}^t \left( \langle p_s, \ell_s \rangle - \ell_{s, j} \right)
    = \sum_{s = t_0}^t \left( \langle w_s, \tilde{\ell}_s \rangle - \tilde{\ell}_{s, m} \right)
    + \sum_{s = t_0}^t \left( \langle q^{(m)}_s, \ell_s \rangle - \ell_{s, j} \right).
\]
Since $H^{m-1} \le t \le H^m$, the upper bound $t \le H^m$ ensures $t \in I_m$ so Lemma~\ref{lem:base_expert_regret} applies over $[t_0, t]$, while the lower bound $H^{m-1} \le t$ gives $H^m = H \cdot H^{m-1} \le (1 + \varepsilon) t$; thus the second sum is at most $\sqrt{\frac{H^m \ln n}{2}} \le \sqrt{1 + \varepsilon} \sqrt{\frac{t \ln n}{2}}$. \rev{Before invoking Theorem~\ref{thm:dynamic_mwu} we verify its hypothesis that $S_{\mathrm{Surv},s} = S_s \cap S_{s+1} \ne \emptyset$ for every $s$: given $s$, choose $m'$ with $H^{m'-1} \le s+1 \le H^{m'}$ (Lemma~\ref{lem:active_existence}); then $\lfloor \delta H^{m'} \rfloor \le \delta H^{m'} < H^{m'-1} \le s+1$ because $\delta(1+\varepsilon) < 1$, so $t_{\mathrm{start}}(m') \le s$ and $t_{\mathrm{end}}(m') \ge s+1$, whence $m' \in S_s \cap S_{s+1}$.} By Theorem~\ref{thm:dynamic_mwu} together with Lemmas~\ref{lem:pool_size_M} and~\ref{lem:arrival_bound_K} (using $\eta_s = \sqrt{\frac{\ln M}{s}}$, $a = \frac{1}{4 M}$, $c = \frac{1}{4}$, $a K = \frac{1}{4}$, $\frac{1}{\eta_{t_0}} \le \frac{1}{\eta_t}$, and $\sum_{s=t_0}^t \eta_s \le 2 \sqrt{t \ln M}$), the first sum is bounded by
\[
    \frac{\ln(4 M)}{\eta_{t_0}} + \left( \frac{1}{\eta_t} - \frac{1}{\eta_{t_0}} \right) \ln M + \frac{1/4}{(3/4) \eta_t} + \frac{1}{8} \sum_{s = t_0}^t \eta_s
    \le \frac{\ln M + \ln 4 + \frac{1}{3}}{\eta_t} + \frac{\sqrt{t \ln M}}{4}
    \le 3 \sqrt{t \ln M},
\]
where we used $\ln M \ge \ln 3 > 1$ (so $\frac{\ln M + \ln 4 + 1/3}{\sqrt{\ln M}} + \frac{\sqrt{\ln M}}{4} \le \left(\frac{19}{12} + \ln 4\right)\sqrt{\ln M} < 3\sqrt{\ln M}$). \rev{The bound $M \ge 3$ holds because $\delta(1+\varepsilon) < 1$ gives $\frac{\ln(2/\delta)}{\ln(1+\varepsilon)} > \frac{\ln(2(1+\varepsilon))}{\ln(1+\varepsilon)} = 1 + \frac{\ln 2}{\ln(1+\varepsilon)} > 2$ for $\varepsilon \in (0,1)$, so $M = \lceil \frac{\ln(2/\delta)}{\ln(1+\varepsilon)} \rceil + 1 \ge 4$.} Since $\ln(1 + \varepsilon) \ge \varepsilon / 2$ for $\varepsilon \in (0, 1)$, we have $M \le \frac{2 \ln(2/\delta)}{\varepsilon} + 2 \le \frac{4 \ln(2/\delta)}{\varepsilon}$ (as $\delta(1+\varepsilon) < 1$ implies $\frac{\ln(2/\delta)}{\varepsilon} > \frac{\ln(2(1+\varepsilon))}{\varepsilon} \ge \ln 4 > 1$), so $3 \sqrt{t \ln M} \le C_{\varepsilon, \delta} \sqrt{t}$, establishing~\eqref{eq:epoch_regret_bound}.
\end{proof}

\begin{lemma}\label{lem:main_anytime}
For any $\varepsilon \in (0, 1)$ and $\delta \in (0, 1)$ satisfying $\delta (1 + \varepsilon) < 1$, the cumulative regret of the master algorithm at every horizon $T \ge 1$ satisfies
\begin{equation}\label{eq:main_anytime_bound}
    R_T \le \frac{\sqrt{1 + \varepsilon} + C_{\varepsilon, \delta} \sqrt{\frac{2}{\ln n}}}{1 - \sqrt{\delta (1 + \varepsilon)}} \sqrt{\frac{T \ln n}{2}},
\end{equation}
where $C_{\varepsilon, \delta}$ is defined in Lemma~\ref{lem:epoch_regret}.
\end{lemma}

\begin{proof}
Let $\alpha := \frac{\sqrt{1 + \varepsilon} + C_{\varepsilon, \delta} \sqrt{\frac{2}{\ln n}}}{1 - \sqrt{\delta (1 + \varepsilon)}}$. We prove by induction on $T \ge 1$ that $R_T \le \alpha \sqrt{\frac{T \ln n}{2}}$. Let $m \ge 1$ be the expert from Lemma~\ref{lem:active_existence} with $H^{m-1} \le T \le H^m$, and set $t_0 := t_{\mathrm{start}}(m)$. For any base expert $j \in [n]$, splitting $[1, T]$ into $[1, t_0 - 1]$ and $[t_0, T]$ gives
\[
    R_{T, j}
    = \sum_{s = 1}^{t_0 - 1} \left( \langle p_s, \ell_s \rangle - \ell_{s, j} \right)
    + \sum_{s = t_0}^T \left( \langle p_s, \ell_s \rangle - \ell_{s, j} \right).
\]
If $t_0 = 1$, the first sum is empty ($0$). If $t_0 > 1$, then by the inductive hypothesis applied at horizon $t_0 - 1 < T$ and Lemma~\ref{lem:active_existence} ($t_0 - 1 \le \delta (1 + \varepsilon) T$),
\[
    \sum_{s = 1}^{t_0 - 1} \left( \langle p_s, \ell_s \rangle - \ell_{s, j} \right)
    \le R_{t_0 - 1}
    \le \alpha \sqrt{\frac{(t_0 - 1) \ln n}{2}}
    \le \alpha \sqrt{\delta (1 + \varepsilon)} \sqrt{\frac{T \ln n}{2}}.
\]
For the second sum, applying Lemma~\ref{lem:epoch_regret} at time $t = T$ and writing $\sqrt{T} = \sqrt{\frac{2}{\ln n}} \sqrt{\frac{T \ln n}{2}}$ gives
\[
    \sum_{s = t_0}^T \left( \langle p_s, \ell_s \rangle - \ell_{s, j} \right)
    \le \sqrt{1 + \varepsilon} \sqrt{\frac{T \ln n}{2}} + C_{\varepsilon, \delta} \sqrt{T}
    = \left( \sqrt{1 + \varepsilon} + C_{\varepsilon, \delta} \sqrt{\frac{2}{\ln n}} \right) \sqrt{\frac{T \ln n}{2}}.
\]
Adding the two bounds and using the definition of $\alpha$, which satisfies $\sqrt{1 + \varepsilon} + C_{\varepsilon, \delta} \sqrt{\frac{2}{\ln n}} = \alpha \bigl( 1 - \sqrt{\delta (1 + \varepsilon)} \bigr)$, we obtain
\begin{align*}
    R_{T, j}
    &\le \alpha \sqrt{\delta (1 + \varepsilon)} \sqrt{\frac{T \ln n}{2}} + \left( \sqrt{1 + \varepsilon} + C_{\varepsilon, \delta} \sqrt{\frac{2}{\ln n}} \right) \sqrt{\frac{T \ln n}{2}} \\
    &= \left( \alpha \sqrt{\delta (1 + \varepsilon)} + \alpha \left( 1 - \sqrt{\delta (1 + \varepsilon)} \right) \right) \sqrt{\frac{T \ln n}{2}} \\
    &= \alpha \sqrt{\frac{T \ln n}{2}}.
\end{align*}
Taking the maximum over $j \in [n]$ completes the induction and establishes~\eqref{eq:main_anytime_bound}.
\end{proof}

\begin{proof}[Proof of Theorem~\ref{thm:main_intro}]
The bound is asymptotic in $n$, so we may assume $\ln n > 10$ (for $\ln n \le 10$ the algorithm may be run with $\varepsilon = \delta = 1/4$). Set
\[
    \varepsilon := \sqrt{\frac{\ln \ln n}{\ln n}}, \qquad \delta := \varepsilon^3 .
\]
Since $x \mapsto (\ln x)/x$ is decreasing for $x > e$, we have $\varepsilon^2 = \frac{\ln \ln n}{\ln n} < \frac{\ln 10}{10} < \frac{1}{4}$, so $\varepsilon < \frac{1}{2}$ and $\delta < \frac{1}{8}$; hence $\delta(1 + \varepsilon) < \frac{3}{16} < 1$ and the pair $(\varepsilon, \delta)$ is admissible in Lemma~\ref{lem:main_anytime}. The master algorithm of that lemma uses no knowledge of the horizon and satisfies $R_T \le \alpha(n) \sqrt{\frac{T \ln n}{2}}$ for every $T \ge 1$ and every loss sequence, where $\alpha(n)$ depends only on $n$ and is given by
\[
    \alpha(n) := \frac{\sqrt{1 + \varepsilon} + C_{\varepsilon, \delta} \sqrt{\frac{2}{\ln n}}}{1 - \sqrt{\delta (1 + \varepsilon)}} .
\]

It remains to show that $\alpha(n) = 1 + O(\varepsilon)$. Since $\sqrt{\delta(1 + \varepsilon)} = \varepsilon^{3/2} \sqrt{1 + \varepsilon}$, subtracting $1$ from $\alpha(n)$ and dividing by $\varepsilon$ gives
\[
    \frac{\alpha(n) - 1}{\varepsilon}
    = \frac{1}{\varepsilon} \cdot \frac{\bigl(\sqrt{1 + \varepsilon} - 1\bigr) + \varepsilon^{3/2} \sqrt{1 + \varepsilon} + C_{\varepsilon, \delta} \sqrt{\frac{2}{\ln n}}}{1 - \varepsilon^{3/2} \sqrt{1 + \varepsilon}} .
\]
Now let $n \to \infty$, so that $\varepsilon \downarrow 0$. Dividing by $\varepsilon$, the first numerator term satisfies $\bigl(\sqrt{1 + \varepsilon} - 1\bigr)/\varepsilon \to \frac{1}{2}$ and the second satisfies $\varepsilon^{1/2} \sqrt{1 + \varepsilon} \to 0$. For the third, substituting $\delta = \varepsilon^3$,
\[
    C_{\varepsilon, \delta}
    = 3 \sqrt{\ln \left( \frac{4 \ln(2/\delta)}{\varepsilon} \right)}
    = 3 \sqrt{\ln \frac{1}{\varepsilon} + \ln \left( 4 \ln 2 + 12 \ln \frac{1}{\varepsilon} \right)}
    = \bigl(3 + o(1)\bigr) \sqrt{\ln \frac{1}{\varepsilon}} ,
\]
the last step because $\ln\bigl(4 \ln 2 + 12 \ln \frac{1}{\varepsilon}\bigr) = O\bigl(\ln \ln \frac{1}{\varepsilon}\bigr) = o\bigl(\ln \frac{1}{\varepsilon}\bigr)$. By the definition of $\varepsilon$,
\[
    \ln \frac{1}{\varepsilon} = \frac{1}{2} \ln \frac{\ln n}{\ln \ln n} = \frac{1}{2} \bigl( \ln \ln n - \ln \ln \ln n \bigr) = \left( \frac{1}{2} + o(1) \right) \ln \ln n ,
\]
so $C_{\varepsilon, \delta} = \bigl(\frac{3}{\sqrt{2}} + o(1)\bigr)\sqrt{\ln \ln n}$, and therefore
\[
    C_{\varepsilon, \delta} \sqrt{\frac{2}{\ln n}}
    = \bigl(3 + o(1)\bigr) \sqrt{\frac{\ln \ln n}{\ln n}}
    = \bigl(3 + o(1)\bigr) \varepsilon ,
\]
and the third term divided by $\varepsilon$ tends to $3$. The denominator tends to $1$, so $\bigl(\alpha(n) - 1\bigr)/\varepsilon \to \frac{1}{2} + 3 = \frac{7}{2}$. In particular
\[
    \alpha(n) = 1 + O(\varepsilon) = 1 + O\left( \sqrt{\frac{\ln \ln n}{\ln n}} \right) ,
\]
which is the bound asserted in Theorem~\ref{thm:main_intro}.
\end{proof}

\section*{Acknowledgements}
CL would like to thank Victor Portella and Nicholas Harvey for much fruitful discussion on this problem.

\bibliographystyle{plainnat}
\bibliography{refs}

\end{document}